\documentclass[11pt]{article}

\usepackage{fullpage}

\usepackage{url}
\usepackage{smile}
\usepackage{pdfpages}
\usepackage{kpfonts}
\usepackage{dsfont}
\usepackage{tgpagella}
\usepackage{natbib}
\usepackage{caption}
\usepackage[]{hyperref}
\usepackage[left]{lineno}
\usepackage{blindtext}
\usepackage[ruled, vlined, algo2e]{algorithm2e}
\usepackage{physics}

\newcommand\figcaption{\def\@captype{figure}\caption}
\newcommand\tabcaption{\def\@captype{table}\caption}

\DeclareMathAlphabet{\mathsf}{OT1}{cmss}{m}{n}
\SetMathAlphabet{\mathsf}{bold}{OT1}{cmss}{bx}{n}

\makeatletter

\newcommand{\Rmnum}[1]{\expandafter\@slowromancap\romannumeral #1@}
\makeatother

\begin{document}


\title{\huge \bf  Towards Understanding the Importance of Shortcut Connections in Residual Networks}

\date{}

\author{ Tianyi Liu, Minshuo Chen, Mo Zhou, Simon S. Du, Enlu Zhou and Tuo Zhao\thanks{T. Liu, M. Chen, E. Zhou, and T. Zhao are affiliated with School of Industrial and Systems Engineering at Georgia Tech; M. Zhou is now affiliated with CS Department of Duke University; S. S. Du is now affiliated with Institute of Advanced Study; This work is done while M. Zhou is at Peking University and S. S. Du is a Ph.D. student at CMU.  T. Liu and M. Chen contribute equally;  Tuo Zhao is the corresponding author; Email: tourzhao@gatech.edu.}}
\maketitle
\begin{abstract}
Residual Network (ResNet) is undoubtedly a milestone in deep learning. 
ResNet is equipped with shortcut connections between layers, and exhibits efficient training using simple first order algorithms. Despite of the great empirical success, the reason behind is far from being well understood. In this paper, we study a two-layer non-overlapping convolutional ResNet. Training such a network requires solving a non-convex optimization problem with a spurious local optimum. We show, however, that gradient descent combined with proper normalization, avoids being trapped by the spurious local optimum, and converges to a global optimum in polynomial time, when the weight of the first layer is initialized at $0$, and that of the second layer is initialized arbitrarily in a ball. Numerical experiments are provided to support our theory.
\end{abstract}


\section{Introduction}\label{sec:intro}

Neural Networks have revolutionized a variety of real world applications in the past few years, such as computer vision \citep{krizhevsky2012imagenet, goodfellow2014generative, Long_2015_CVPR}, natural language processing \citep{graves2013speech, bahdanau2014neural, young2018recent}, etc. Among different types of networks, Residual Network (ResNet, \citet{he2016deep}) is undoubted a milestone. ResNet is equipped with shortcut connections, which skip layers in the forward step of an input. Similar idea also appears in the Highway Networks \citep{srivastava2015training}, and further inspires densely connected convolutional networks \citep{huang2017densely}.

ResNet owes its great success to a surprisingly efficient training compared to the widely used feedforward Convolutional Neural Networks (CNN, \citet{krizhevsky2012imagenet}). Feedforward CNNs are seldomly used with more than 30 layers in the existing literature. There are experimental results suggest that very deep feedforward CNNs are significantly slow to train, and yield worse performance than their shallow counterparts \citep{he2016deep}. However, simple first order algorithms such as stochastic gradient descent and its variants are able to train ResNet with hundreds of layers, and achieve better performance than the state-of-the-art. For example, ResNet-152 \citep{he2016deep}, consisting of 152 layers, achieves a $19.38\%$ top-1 error on ImageNet. \citet{he2016} also demonstrated a more aggressive ResNet-1001 on the CIFAR-10 data set with 1000 layers. It achieves a $4.92\%$ error --- better than shallower ResNets such as ResNet-$110$.

Despite the great success and popularity of ResNet, the reason why it can be efficiently trained is still largely unknown. One line of research empirically studies ResNet and provides intriguing observations. \citet{veit2016residual}, for example, suggest that ResNet can be viewed as a collection of weakly dependent smaller networks of varying sizes. More interestingly, they reveal that these smaller networks alleviate the vanishing gradient problem. \citet{balduzzi2017shattered} further elaborate on the vanishing gradient problem. They show that the gradient in ResNet only decays sublinearly in contrast to the exponential decay in feedforward neural networks. Recently, \citet{li2018visualizing} visualize the landscape of neural networks, and show that the shortcut connection yields a smoother optimization landscape. In spite of these empirical evidences, rigorous theoretical justifications are seriously lacking. 

Another line of research theoretically investigates ResNet with simple network architectures. \citet{hardt2016identity} show that linear ResNet has no spurious local optima (local optima that yield larger objective values than the global optima). Later, \citet{li2017convergence} study using Stochastic Gradient Descent (SGD) to train a two-layer ResNet with only one unknown layer. They show that the optimization landscape has no spurious local optima and saddle points. They also characterize the local convergence of SGD around the global optimum. These results, however, are often considered to be overoptimistic, due to the oversimplified assumptions.

To better understand ResNet, we study a two-layer non-overlapping convolutional neural network, whose optimization landscape contains a spurious local optimum. Such a network was first studied in \citet{du2017gradient}. Specifically, we consider 
\begin{align}\label{eq:cnn}
g(v, a, Z) = a^\top \sigma\left(Z^\top v\right),
\end{align}
where $Z \in \RR^{p \times k}$ is an input, $a \in \RR^k, v \in \RR^p$ are the output weight and the convolutional weight, respectively, and $\sigma$ is the element-wise ReLU activation. Since the ReLU activation is positive homogeneous, the weights $a$ and $v$ can arbitrarily scale with each other. Thus, we  impose the assumption $\norm{v}_2=1$  to make the neural network identifiable. We further decompose $v = \frac{\mathds{1}}{\sqrt{p}} + w$ with $\mathds{1}$ being a vector of $1$'s in $\RR^p$, and rewrite \eqref{eq:cnn} as
\begin{align}\label{eq:resnet}
& f(w,a,Z) =a^\top\sigma\left({Z^\top\left(\frac{\mathds{1}}{\sqrt{p}}+w\right)}\right),
\end{align}
Here $\frac{\mathds{1}}{\sqrt{p}}$ represents the average pooling shortcut connection, which allows a direct interaction between the input $Z$ and the output weight $a$.

We investigate the convergence of training ResNet by considering a realizable case. Specifically, the training data is generated from a teacher network with true parameters $a^*$, $v^*$ with $\norm{v^*}_2 = 1$. We aim to recover the teacher neural network using a student network defined in \eqref{eq:resnet} by solving an optimization problem:
\begin{align}\label{eq:obj}
(\hat{w}, \hat{a}) = \argmin_{w, a} \frac{1}{2} \EE_Z \left[f(w, a, Z) - g(v^*, a^*, Z)\right]^2,
\end{align}
where $Z$ is independent Gaussian input. Although largely simplified, \eqref{eq:obj} is nonconvex and possesses a nuisance --- There exists a spurious local optimum (see an explicit characterization in Section \ref{sec:model}). Early work, \citet{du2017gradient}, show that when the student network has the same architecture as the teacher network, GD with random initialization can be trapped in a spurious local optimum with a constant probability\footnote{The probability is bounded between $1/4$ and $3/4$. Numerical experiments show that this probability can be as bad as $1/2$ with the worst configuration of $a, v$.}. A natural question here is
\begin{center}

\emph{Does the shortcut connection ease the training?}

\end{center}
This paper suggests a positive answer: When initialized with $w = 0$ and $a$ arbitrarily in a ball, GD with proper normalization converges to a global optimum of \eqref{eq:obj} in polynomial time, under the assumption that $(v^*)^\top \left(\frac{\mathds{1}}{\sqrt{p}}\right)$ is close to $1$. Such an assumption requires that there exists a $w^*$ of relatively small magnitude, such that $v^* = \frac{\mathds{1}}{\sqrt{p}} + w^*$. This assumption is supported by both empirical and theoretical evidences. Specifically, the experiments in \citet{li2016demystifying} and \citet{yu2018learning}, show that the weight in well-trained deep ResNet has a small magnitude, and the weight for each layer has vanishing norm as the depth tends to infinity. \citet{hardt2016identity} suggest that, when using linear ResNet to approximate linear transformations, the norm of the weight in each layer scales as $O(1/D)$ with $D$ being the depth. \citet{bartlett2018representing} further show that deep nonlinear ResNet, with the norm of the weight of order $O(\log D / D)$, is sufficient to express differentiable functions under certain regularity conditions. These results motivate us to assume $w^*$ is relatively small.

Our analysis shows that the convergence of GD exhibits 2 stages. Specifically, our initialization guarantees $w$ is sufficiently away from the spurious local optimum. In the first stage, with proper step sizes, we show that the shortcut connection helps the algorithm avoid being attracted by the spurious local optima. Meanwhile, the shortcut connection guides the algorithm to evolve towards a global optimum. In the second stage, the algorithm enters the basin of attraction of the global optimum. With properly chosen step sizes, $w$ and $a$ jointly converge to the global optimum.

Our analysis thus explains why ResNet benefits training, when the weights are simply initialized at zero \citep{li2016demystifying}, or using the Fixup initialization in \citet{zhang2019fixup}. We remark that our choice of step sizes is also related to learning rate warmup \citep{goyal2017accurate}, and other learning rate schemes for more efficient training of neural networks \citep{smith2017cyclical, smith2018super}. We refer readers to Section \ref{sec:discuss} for a more detailed discussion.

\textbf{Notations:} Given a vector $v = (v_1, \dots, v_m)^\top \in \mathbb{R}^m$, we denote the Euclidean norm $\lVert v \rVert_2^2 = v^\top v$. Given two vectors $u, v \in \RR^d$, we denote the angle between them as $\angle (u, v) = \arccos \frac{u^\top v}{\norm{u}_2 \norm{v}_2}$, and the inner product as $\langle u, v \rangle = u^\top v$. We denote $\mathds{1} \in \RR^d$ as the vector of all the entries being $1$. We also denote $\BB_0(r) \in \RR^d$ as the Euclidean ball centered at $0$ with radius $r$.

\section{Model and Algorithm}\label{sec:model}
\begin{minipage}{0.66\textwidth}
\paragraph{Model.} We consider the realizable setting where the label is generated from a noiseless teacher network in the following form
\begin{align}\label{teacher}
	\textstyle g(v^*,a^*,Z) = \sum_{j=1}^ka^*_j\sigma\left(Z_j^\top v^*\right).
\end{align}
Here $v^*,a^*,Z_j$'s are the true convolutional weight, true output weight, and input. $\sigma$ denotes the element-wise ReLU activation.

Our student network is defined in \eqref{eq:resnet}. For notational convenience, we expand the second layer and rewrite \eqref{eq:resnet} as
\begin{align} \label{fprime}
	&\textstyle f(w,a,Z) = \sum_{j=1}^ka_j\sigma\left({Z_j^\top(\mathds{1}/\sqrt{p}+w)}\right),
\end{align}
where  $w\in \mathbb{R}^p$, $a_j \in \mathbb{R}$, and $Z_j \in \mathbb{R}^{p}$ for all $j=1,2, \dots ,k$. We assume the input data $Z_j$'s are  identically independently sampled from $\cN(0, I).$ Note that the above network is not identifiable, because of  the positive homogeneity of the ReLU function, that is
\end{minipage}
\hspace{0.02\textwidth}
\begin{minipage}{0.31\textwidth}
\centering
\includegraphics[width = 0.85\textwidth]{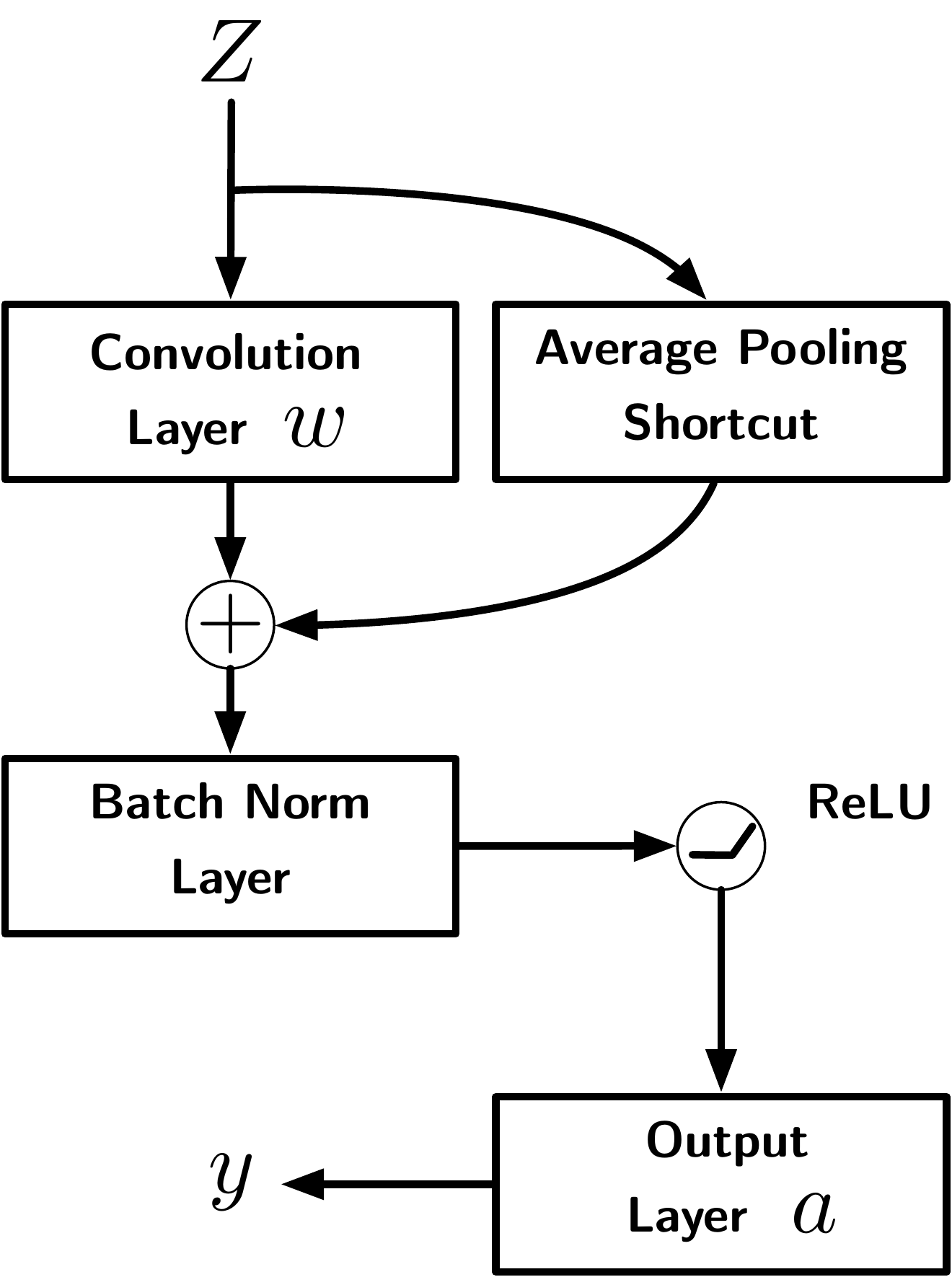}
\captionof{figure}{The non-overlapping two layer residual network with normalization layer.
}
\label{fig:ResNet2}
\end{minipage}
$\mathds{1}\sqrt{p}+w$ and $a$ can scale with each other by any positive constant without changing the output value. Thus, to achieve identifiability, instead of  \eqref{fprime}, we propose to train the following student network,
\begin{align}\label{student-network} 
&\textstyle f(w,a,Z) = \sum_{j=1}^ka_j\sigma\left(Z_j^\top \frac{\mathds{1}/\sqrt{p}+w}{\norm{\mathds{1}/\sqrt{p}+w}_2}\right).
\end{align}
An illustration of \eqref{student-network} is provided in Figure \ref{fig:ResNet2}.  An example  of the teacher network \eqref{teacher} and the student network  \eqref{student-network}  is shown in Figure \ref{fig:ResNet1}.
\begin{figure}
\centering
\includegraphics[width = 0.85\textwidth]{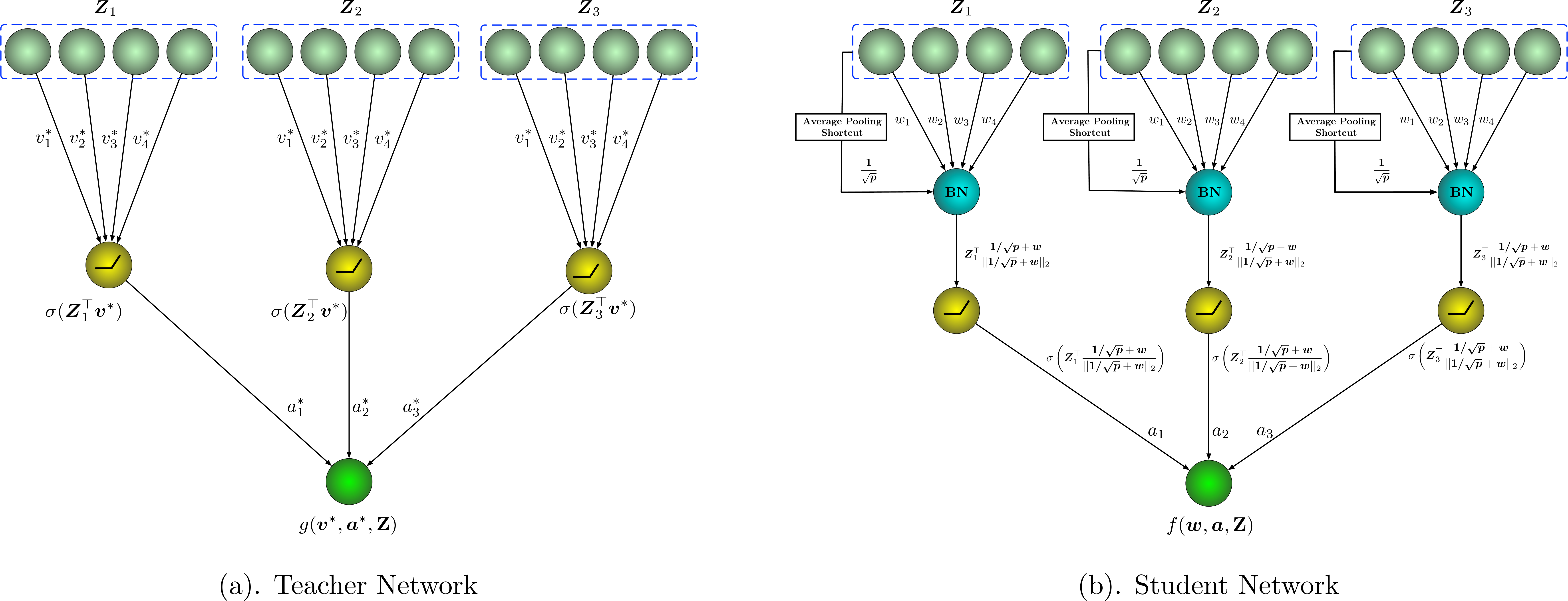}
\caption{Illustrative examples of the teacher and student networks with $k=3$ and $p=4$.  BN notes batch normalization. }
\label{fig:ResNet1}
\end{figure}
 We then recover $(v^*, a^*)$ of our teacher network by solving a nonconvex optimization problem
\begin{align}\label{optimization1}
&\min_{w,a}\cL(w,a) = \frac{1}{2}\EE_Z[g(v^*,a^*,Z)-f(w,a,Z)]^2.
\end{align}
Recall that we assume $\norm{v^*}_2 = 1$. One can easily verify that \eqref{optimization1} has global optima and spurious local optima. The characterization is analogous to \citet{du2017gradient}, although the objective is different.
\begin{proposition}\label{prop:spurious}
For any constant $\alpha>0,$ $(w, a)$ is a global optimum of \eqref{optimization1}, if $\frac{\mathds{1}}{\sqrt{p} } + w= \alpha v^*$  and $a = a^*;$ 
$(w, a)$ is a spurious local optimum of \eqref{optimization1}, if ${\frac{\mathds{1}}{\sqrt{p} } + w}= -\alpha v^*$ and $a = (\mathds{1}\mathds{1}^\top+(\pi-1)I)^{-1}(\mathds{1}\mathds{1}^\top-I)a^*.$
\end{proposition}
The proof is adapted from \citet{du2017gradient}, and the details are provided in Appendix \ref{pf:spurious}. 

Now we formalize the assumption on $v^*$ in Section \ref{sec:intro}, which is supported by the theoretical and empirical evidence in \citet{li2016demystifying, yu2018learning, hardt2016identity, bartlett2018representing}.
\begin{assumption}[Shortcut Prior]\label{assump1}
There exists a $w^*$ with $\norm{w^*}_2\leq 1,$ such that $v^* = w^*+\frac{\mathds{1}}{\sqrt{p} }.$
\end{assumption}
Assumption \ref{assump1} implies $(\frac{\mathds{1}}{\sqrt{p} })^\top v^* \geq 1/2$. We remark that our analysis actually applies to any $w^*$ satisfying $\norm{w^*}_2\leq c$ for any positive constant $c\in(0,\sqrt{2})$. Here we consider $\norm{w^*}_2 \leq 1$ to ease the presentation.
Throughout the rest of the paper, we assume this assumption holds true.

\paragraph{GD with Normalization.} 
We solve the optimization problem \eqref{optimization1} by gradient descent. Specifically, at the $(t+1)$-th iteration, we compute
\begin{align}
\tilde{w}_{t+1} &= w_t - \eta_w\nabla_w\cL(w_t,a_t), \notag\\
w_{t+1} &= \frac{\frac{\mathds{1}}{\sqrt{p}} +\tilde{w}_{t+1}}{\norm{\frac{\mathds{1}}{\sqrt{p}}+\tilde{w}_{t+1}}_2}-\frac{\mathds{1}}{\sqrt{p}},\label{normalization}\\
a_{t+1} &= a_t - \eta_a\nabla_a\cL(w_t,a_t).\notag
\end{align}
Note that we normalize $\frac{\mathds{1}}{\sqrt{p} } + w$ in \eqref{normalization}, which essentially guarantees $\Var\left(Z_j^\top\left(\frac{\mathds{1}}{\sqrt{p} }+w_{t+1}\right)\right)=1.$ 
As $Z_j$ is sampled from $N(0,I)$, we further have $\EE\left(Z_j^\top\left(\frac{\mathds{1}}{\sqrt{p} }+w_{t+1}\right)\right)=0$. The normalization step in \eqref{normalization} can be viewed as a population version of the widely used batch normalization trick to accelerate the training of neural networks \citep{ioffe2015batch}. Moreover,  \eqref{optimization1} has one unique optimal solution under such a normalization. Specifically, $(w^*, a^*)$ is the unique global optimum, and $(\bar{w}, \bar{a})$ is the only spurious local optimum along the solution path, where $\bar{w} = - (\frac{\mathds{1}}{\sqrt{p} }) - v^*$ and $\bar{a} = (\mathds{1}\mathds{1}^\top+(\pi-1)I)^{-1}(\mathds{1}\mathds{1}^\top-I)a^*$. 


\begin{figure}
\centering
\includegraphics[width = 0.8\textwidth]{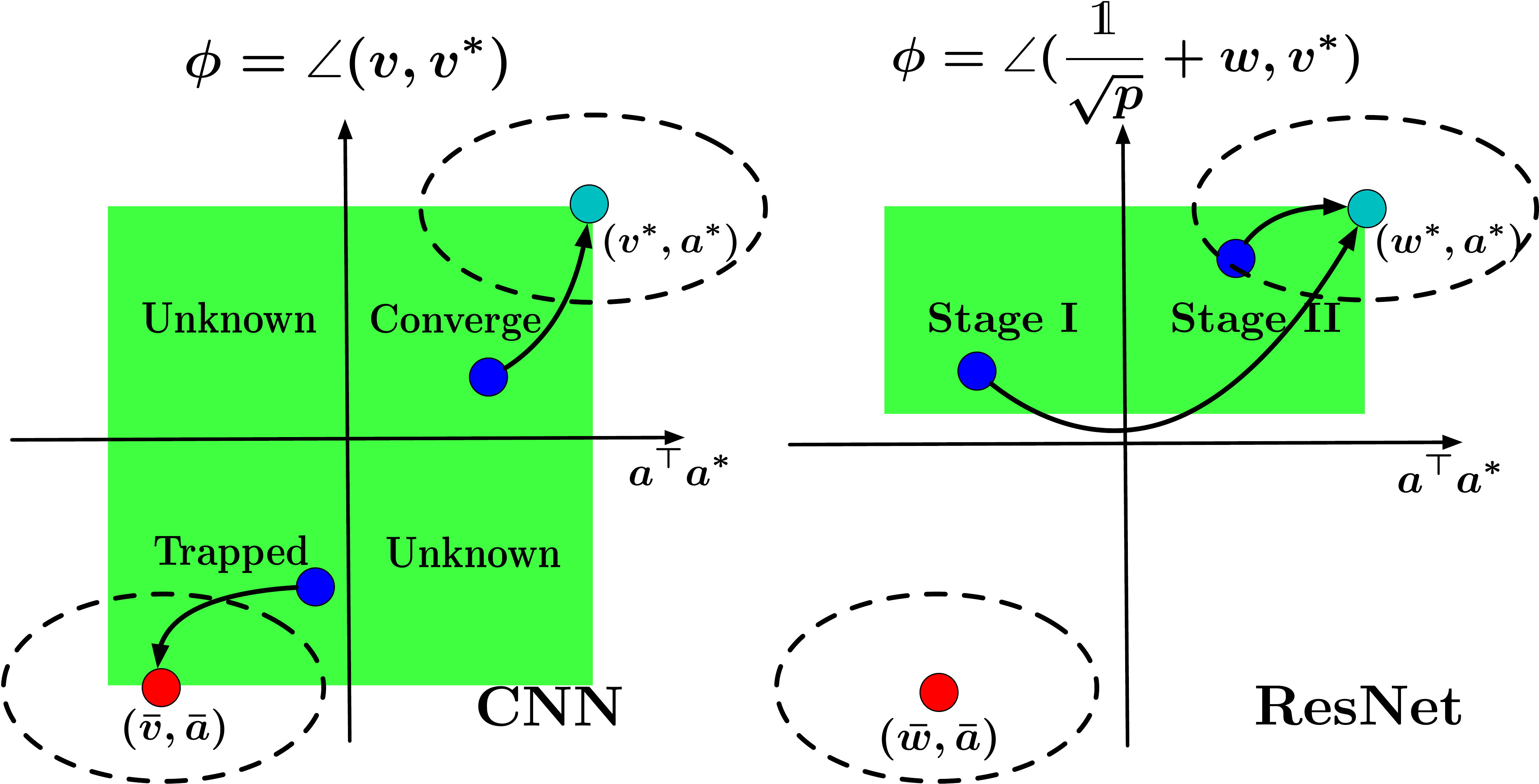}
\caption{The left panel shows random initialization on feedforward CNN can be trapped in the spurious local optimum with probability at least $1/4$ \citep{du2017gradient}. The right panel demonstrates: 1). Under the shortcut prior, our initialization of $(w, a)$ avoids starting near the spurious local optimum; 2). Convergence of GD exhibits two stages (I. improvement of $a$ and avoiding being attracted by $(\bar {w},\bar{a})$ II. joint convergence).}
\label{fig:converge}
\end{figure}
We initialize our algorithm at $(w_0,a_0)$ satisfying: $w_0=0$ and $a_0 \in \mathbb{B}_0\left(\frac{|\mathds{1}^\top a^*| }{ \sqrt{k}}\right).$ We set $a_0$ with a magnitude of $O\left(\frac{1}{\sqrt{k}}\right)$ to match common initialization techniques \citep{glorot2010understanding, lecun2012efficient, he2015delving}. We highlight that our algorithm starts with an arbitrary initialization on $a$, which is different from random initialization. The step sizes $\eta_a$ and $\eta_w$ will be specified later in our analysis.


\section{Convergence Analysis}\label{sec:sketch}

We characterize the algorithmic behavior of the gradient descent algorithm. 
Our analysis shows that under  Assumption \ref{assump1},  the convergence of  GD exhibits two stages. In the first stage, the algorithm avoids being trapped by the spurious local optimum. Given the algorithm is sufficiently away from the spurious local optima, the algorithm enters the basin of attraction of the global optimum and finally converge to it.


To present our main result, we begin with some notations. Denote $$\phi_t = \angle\left(\frac{\mathds{1}}{\sqrt{p} } +w_t,\frac{\mathds{1}}{\sqrt{p} }+ w^*\right)$$as the angle between $\frac{\mathds{1}}{\sqrt{p} }+w_t$ and the ground truth at the $t$-th iteration. Throughout the rest of the paper, we assume $\norm{{a}^*}_2$ is a constant. The notation $\tilde{O}(\cdot)$ hides $\mathrm{poly}(\norm{{a}^*}_2),$  $\mathrm{poly}(\frac{1}{\norm{{a}^*}_2})$, and $\mathrm{polylog}(\norm{{a}^*}_2)$ factors. Then we state the convergence of GD in the following theorem.
\begin{theorem}[Main Results]
Let the GD algorithm defined in Section \ref{sec:model} be initialized with $w_0=0$ {and arbitrary} $a_0 \in \mathbb{B}_0 (\frac{|\mathds{1}^\top a^*| }{ \sqrt{k}}).$ Then the algorithm converges in two stages:

\noindent\textbf{Stage I: Avoid the spurious local optimum} (Theorem \ref{stage1_a}): We choose $$\eta_a=O\left(\frac{1}{ k^2}\right)~\text{and }\eta_w=\tilde O\left(\frac{1}{k^4}\right).$$ Then there exists $T_1 = \tilde{O}(\frac{1}{\eta_a})$, such that $$m\leq a_{T_1}^\top{a}^* \leq M\text{ and }\phi_{T_1}\leq \frac{5}{12}\pi$$ hold for some constants $M > m > 0$.

\noindent\textbf{Stage II: Converge to the global optimum} (Theorem \ref{converge}): After $T_1$ iterations, we restart the counter, and choose $$\eta=\eta_a=\eta_w=\tilde O\left(\frac{1}{k^2}\right).$$ Then for any $\delta>0$, any $t\geq T_2=\tilde O(\frac{1}{\eta}\log\frac{1}{\delta})$, we have $$\norm{w_{t}-w^*}_2^2\leq \delta~\text{and}~\norm{a_{t}-a^*}_2^2\leq 5\delta.$$


\end{theorem}
Note that the set $\{(w_t, a_t) ~|~ a_t^\top a^* \in [m, M], \phi_t \leq \frac{5 \pi }{ 12} \}$ belongs to be the basin of attraction around the global optimum (Lemma \ref{lem_pd_w}), where certain regularity condition (partial dissipativity) guides the algorithm toward the global optimum. Hence, after the algorithm enters the second stage, we increase the step size $\eta_w$ of $w$ for a faster convergence. Figure \ref{fig:converge} demonstrates the initialization of $(w, a)$, and the convergence of GD both on CNN in \citet{du2017gradient} and our ResNet model.

We  start our convergence analysis with the definition of partial dissipativity for  $\cL$.
\begin{definition}[Partial Dissipativity]\label{SC}
Given any $\delta \geq 0$ and a constant $c\geq0$, $\nabla_w\cL$ is $(c,\delta)$-partially dissipative with respect to $w^*$ in a set $\cK_\delta$, if for every $(w,a)\in \cK_\delta,$ we have
\begin{align*}
	\langle-\nabla_w\cL(w,a),w^*-w\rangle\geq c\norm{w-w^*}_2^2-\delta;
\end{align*}
$\nabla_a\cL$ is $(c,\delta)$-partially dissipative with respect to $a^*$ in a set $\cA_\delta$, if for every $(w,a)\in \cA_\delta,$ we have
\begin{align*}
 \langle-\nabla_a\cL(w,a),a^*-a\rangle\geq c\norm{a-a^*}_2^2-\delta.
\end{align*}
Moreover, If $\cK_\delta\cap\cA_\delta\neq \emptyset,$  $\nabla \cL$ is $(c,2\delta)$-jointly dissipative with respect to $(w^*,a^*)$ in $\cK_\delta\cap\cA_\delta,$ i.e., for every $(w,a)\in \cK_\delta\cap\cA_\delta,$ we have
\begin{align*}
\langle-\nabla_w\cL(w,a),w^*-w\rangle+\langle-\nabla_a\cL(w,a),a^*-a\rangle\geq c(\norm{w-w^*}_2^2+\norm{a-a^*}_2^2)-2\delta.
\end{align*} 
\end{definition}
The concept of dissipativity is originally used in dynamical systems \citep{barrera2015thermalisation}, and is defined for general operators.  It suffices to instantiate the concept to gradients here for our convergence analysis.  The partial dissipativity for perturbed gradients is used in \cite{zhou2019toward} to study the convergence behavior of Perturbed GD. The variational coherence studied in \cite{zhou2017stochastic} and one point convexity studied in \cite{li2017convergence} can be viewed as  special examples of partial dissipativity.

\subsection{Stage I: Avoid the Spurious Local Optimum}

We first show with properly chosen step sizes, GD algorithm can avoid being trapped by the spurious local optimum.
We propose to update $w, a$ using different step sizes. We formalize our result in the following theorem.
\begin{theorem}\label{stage1_a}
Initialize with arbitrary $a_0 \in \mathbb{B}_0 \left(\frac{|\mathds{1}^\top a^*| }{ \sqrt{k}}\right)~\text{and}~w_0=0$. We choose step sizes $$\eta_a = \frac{\pi}{20(k+\pi-1)^2}=O\left(\frac{1}{k^2}\right), \quad \text{and}\quad \eta_w=C\norm{a^*}_2^2\eta_a^2=\tilde O(\eta_a^2)$$ for some constant $C>0.$ Then, we have 
\begin{align}
\label{stage1_a_result}
\phi_t\leq \frac{5\pi}{12}\quad \textrm{and}\quad 0\leq m\leq a_t^\top{a}^* \leq M,
\end{align} 
for all $t \in [T_1, T]$, where $$T_1=\tilde O\left(\frac{1}{\eta_a}\right),~~T=O\left(\frac{1}{\eta_a^2}\right),~~ m=\frac{1}{5}\norm{a^*}_2^2, ~~\textrm{and}~~M=3\norm{a^*}_2^2+2\left(\mathds{1}^\top a^*\right)^2.$$
\end{theorem}
\begin{proof}[Proof Sketch]
Due to the space limit, we only provide a proof sketch here. The detailed proof is deferred to Appendix \ref{pf_stage1_a}. 
 We prove the two arguments in \eqref{stage1_a_result} in order.
Before that, we first show our initialization scheme guarantees an important bound on $a,$ as stated in the following lemma. \begin{lemma}\label{lem_sum_a}
	Given $a_0 \in \mathbb{B}_0 \left(\frac{|\mathds{1}^\top a^*| }{ \sqrt{k}}\right),$ we choose $$\eta_a\leq \frac{2\pi}{k+\pi-1}.$$
	Then for any $t > 0$,
	\begin{align}\label{eq_sum_a}
-3\left(\mathds{1}^\top a^*\right)^2\leq \mathds{1}^\top a^*\mathds{1}^\top a_{t} - \left(\mathds{1}^\top a^*\right)^2
	\leq 0.
	\end{align}
\end{lemma} 

 Under  the shortcut prior assumption \ref{assump1} that $w_0$ is close to $w^*$, the update of $w$ should be more conservative to provide enough accuracy for $a$ to make progress. Based on Lemma \ref{lem_sum_a}, the next lemma shows that when $\eta_w$ is small enough, $\phi_t$ stays acute ($\phi_t<\frac{\pi}{2}$), i.e., $w$ is sufficiently away from $\bar w= - \frac{\mathds{1}}{\sqrt{p}} - v^*$ .
\begin{lemma}\label{lem_small}
	Given $w_0=0~\text{and}~a_0 \in \mathbb{B}_0 \left(\frac{|\mathds{1}^\top a^*| }{\sqrt{k}}\right),$ we choose $$\eta_a<\frac{2\pi}{k+\pi-1}
\text{	and }\eta_w=C\norm{a^*}_2^2\eta_a^2=\tilde O(\eta_a^2)$$ for some absolute constant $C>0$. Then for all $t\leq T=O\left(\frac{1}{\eta_a^2}\right)$,
	\begin{align}\label{lem_phi_acute}
	\phi_t\leq \frac{5\pi}{ 12}.
	\end{align}
\end{lemma}
We want to remark that  \eqref{eq_sum_a} and  \eqref{lem_phi_acute} are two  of the key conditions that define the partially dissipative region of $\nabla_a\cL$ , as shown in the following lemma.
\begin{lemma}\label{thm_pd_a}
	For any $(w,a)\in\cA,$ $\nabla_a \cL$ satisfies
	\begin{align}\label{eq_pda}
	\langle-\nabla_a\cL(w,a),{a}^*-a\rangle \geq \frac{1}{10\pi} \norm{a-{a}^*}_2^2,
	\end{align}
	where 
	\begin{align*}\cA=\bigg\{(w,a)~\big| a^\top a^*\leq {\frac{1}{20}}\norm{a^*}_2^2  \text{ or}~
	\norm{a-\frac{a^*}{2}}_2^2\geq\norm{a^*}_2^2,~ &\norm{w+\frac{\mathds{1}}{\sqrt{p}}}_2=1,\\&\phi \leq \frac{5}{12}\pi,
	~-3(\mathds{1}^\top{a}^*)^2\leq \mathds{1}^\top {a}^*\mathds{1}^\top a - (\mathds{1}^\top {a}^*)^2
	\leq  0\bigg\}.
	\end{align*}
\end{lemma}
Please refer to Appendix \ref{pf_thm_pd} for a detailed proof. Note that with arbitrary initialization of $a$,  $a^\top a^*\leq {\frac{1}{20}}\norm{a^*}_2^2$ or $	\norm{a-a^*/2}_2^2\geq\norm{a^*}_2^2$ possibly holds at $a_0.$ In this case,  $(w_0,a_0)$ falls in $\cA,$  and \eqref{eq_pda} ensures the improvement of $a.$ 

\begin{lemma}\label{lem_escape}
Given $(w_0,a_0)\in \cA,$  we choose $$\eta_a<\frac{\pi}{20(k+\pi-1)^2}.$$ Then there exists $\tau_{11}=O\left(\frac{1}{\eta_a}\right),$ such that
 $$\frac{1}{20}\norm{a^*}_2^2\leq a_{\tau_{11}}^\top a^*\leq 2\norm{a^*}_2^2.$$
\end{lemma}
One can easily verify that $a^\top a^*\leq 2\norm{a^*}_2^2$ holds for any $a\in \mathbb{B}_0 \left(\frac{|\mathds{1}^\top a^*| }{\sqrt{k}}\right).$ Together with Lemma \ref{lem_escape}, we claim that even with arbitrary initialization, the iterates can always enter the region with $a^\top a^*$ positive and bounded in polynomial time. The next lemma shows that with proper chosen step sizes,  $a^\top a^*$ stays positive and bounded. 
\begin{lemma}	\label{lem_inner_a} 
Suppose $\frac{1}{20}\norm{a^*}_2^2\leq a_0^\top a^*\leq 2\norm{a^*}_2^2,$
$\phi_t\leq \frac{5}{12}\pi$, and $ -3\left(\mathds{1}^\top a^*\right)^2
\leq \mathds{1}^\top a^*\mathds{1}^\top a_{t} - \left(\mathds{1}^\top a^*\right)^2
\leq  0$
holds for all $t$.
Choose $$\eta_a < \frac{2\pi}{\pi-1},$$ then we have for all $t\geq \tau_{12}=\tilde O\left(\frac{1 }{ \eta_a}\right),$ 
$$
\frac{1}{5}\norm{a^*}_2^2 \leq
a_{t}^\top a^*
\leq 3\norm{a^*}_2^2+2\left(\mathds{1}^\top a^*\right)^2.
$$
\end{lemma}
Take $T_1=\tau_{11}+\tau_{12},$ and we complete the proof.
\end{proof}
In Theorem \ref{stage1_a},  we choose a conservative $\eta_w$. This  brings two benefits to the training process: 1). $w$ stays away from $\bar{w}$. The update on $w$ is quite limited, since $\eta_w$ is small. Hence, $w$ is kept sufficiently away from $\bar{w}$, even if $w$ moves towards $\bar{w}$ in every iteration); 2). $a$ continuously updates toward $a^*.$ 

Theorem \ref{stage1_a} ensures that under the shortcut prior, GD with adaptive step sizes can successfully overcome the optimization challenge early in training, i.e., the iterate is sufficiently away from the spurious local optima at the end of Stage I. Meanwhile,  \eqref{stage1_a_result} actually demonstrates that the algorithm enters the basin of attraction of the global optimum, and  we  next show the convergence of GD.

\subsection{Stage II: Converge to the Global Optimum}

Recall that in the previous stage, we use a conservative step size $\eta_w$ to avoid being trapped by the spurious local optimum. However, the small step size $\eta_w$ slows down the convergence of $w$ in the basin of attraction of the global optimum.  Now we  choose larger step sizes to accelerate the convergence. The following theorem shows that, after Stage I, we can use a larger $\eta_w,$ while the results in Theorem \ref{stage1_a} still hold, i.e., the iterate stays in the basin of attraction of $(w^*,a^*)$.
\begin{theorem}\label{large_etaw}
We restart the counter of time. Suppose $m\leq a_0^\top a^*\leq M, $ and $\phi_0\leq \frac{5}{12}\pi.$ We choose $$\eta_w\leq  \frac{ m}{M^2}=\tilde O \left(\frac{1}{k^2}\right)\text{ and }\eta_a<\frac{2\pi}{k+\pi-1}.$$
Then  for all $t>0$, we have 
$$ \phi_t\leq \frac{5\pi}{12}\quad\textrm{and}\quad 0\leq m\leq a_t^\top{a}^* \leq M
.$$
\end{theorem}
\begin{proof}[Proof Sketch]
To prove the first argument, we need the partial dissipativity  of $\nabla_w \cL$.
\begin{lemma}\label{lem_pd_w}
 For any $m>0$, $\nabla_w \cL$ satisfies
	\begin{align*}
	\langle-\nabla_w\cL(w,a),{w}^*-w\rangle \geq \frac{m}{8}\norm{w-{w}^*}_2^2,
	\end{align*}
	for any $(w,a)\in\mathcal{K}_{m}$, where
	\begin{align*}
	\mathcal{K}_{m}=\left\{(w,a)~\big|~ a^\top {a}^*\geq m,~ \left(w+\frac{\mathds{1}}{\sqrt{p} }\right)^\top v^*\geq 0,~\norm{w+\frac{\mathds{1}}{\sqrt{p} }}_2=1\right\}.
	\end{align*}
\end{lemma}
This condition ensures that when $a^\top {a}^*$ is positive, $w$ always makes positive progress towards $w^*,$ or equivalently $\phi_t$ decreasing.  We need not worry about $\phi_t$ getting obtuse, and thus a larger step size $\eta_w$ can be adopted. The second argument can be proved following similar lines to Lemma \ref{lem_inner_a}.  Please see Appendix \ref{pf_large_etaw} for more details.
\end{proof}
Now we are ready to show the convergence of our GD algorithm. Note that  Theorem \ref{large_etaw} and Lemma \ref{lem_pd_w} together show that the iterate stays in the partially dissipative region $\cK_w,$ which leads to the convergence of $w.$ Moreover, as shown in the following lemma, when $w$ is accurate enough, the partial  gradient with respect to $a$ enjoys partial dissipativity.
\begin{lemma}\label{new_pda}
	For any $\delta>0,$ $\nabla_a \cL$ satisfies
	\begin{align*}
	\langle-\nabla_a\cL\left(w,a\right),a^*-a\rangle\geq \frac{\pi-1}{2\pi}\norm{a-a^*}_2^2 -\frac{1}{5}\delta,
	\end{align*}
	for any $(w,a)\in\cA_{m,M,\delta}$, where 
	\begin{align*}
	{\cA}_{m,M,\delta}=\left\{(w,a)~\big|~ a^\top {a}^*\in[m,M],~ \norm{{w}-w^*}_2^2\leq \delta,~\norm{w+\frac{\mathds{1}}{\sqrt{p} }}_2=1\right\}.
	\end{align*}
\end{lemma}
As a direct result, $a$ converges to $a^*$. The next theorem formalize the above discussion. 

\begin{theorem}[Convergence]\label{converge}
	Suppose $\frac{1}{5}\norm{a^*}_2^2=m\leq a_{t}^\top a^*
	\leq M=3\norm{a^*}_2^2+2\left(\mathds{1}^\top a^*\right)^2$ hold for all $t>0.$ For any $\delta>0,$ choose $$\eta_a=\eta_w=\eta=\min\left\{ \frac{m}{2M^2},\frac{5\pi^2}{4\left(k+\pi-1\right)^2}\right\}=\tilde O\left(\frac{1}{k^2}\right),$$then we have  
	$$\norm{w_{t}-w^*}_2^2\leq \delta~~\text{and}~~\norm{a_{t}-a^*}_2^2\leq 5\delta$$
for any $t\geq T_2=\tilde O\left(\frac{1}{\eta}\log\frac{1}{\delta}\right).$ 
\end{theorem}
\begin{proof}[Proof Sketch]
The detailed proof is provided in Appendix \ref{pf_converge}. Our proof  relies on the partial dissipativity of $\nabla_w \cL$   (Lemma \ref{lem_pd_w}) and that of $\nabla_a \cL$  (Lemma \ref{new_pda}).

Note that the partial dissipative region  ${\cA}_{m,M,\delta},$ depends on the precision of $w.$ Thus, we first  show the convergence of $w.$
\begin{lemma}[Convergence of $w_t$]\label{converge_w}
Suppose $\frac{1}{5}\norm{a^*}_2^2=m\leq a_{t}^\top a^*
	\leq M=3\norm{a^*}_2^2+4\left(\mathds{1}^\top a^*\right)^2$  hold for all $t>0.$ For any $\delta>0,$ choose$$\eta\leq\frac{m}{2M^2}=\tilde O\left(\frac{1}{k^2}\right),$$ then we have  
	$$\norm{w_{t}-w^*}_2^2\leq \delta$$
for any $t\geq\tau_{21}=\frac{4}{m\eta}\log\frac{4}{\delta}=\tilde O\left(\frac{1}{\eta}\log\frac{1}{\delta}\right).$
\end{lemma}
Lemma \ref{converge_w} implies that after $\tau_{21}$ iterations, the algorithm enters  ${\cA}_{m,M,\delta}.$ Then we show the convergence property of $a$  in next lemma.
\begin{lemma}[Convergence of $a_t$]\label{converge_a}
Suppose $m\leq a_{t}^\top a^*
\leq M$ and $\norm{{w}_{t}-w^*}_2^2\leq \delta$ holds for all $t.$ We choose $$\eta\leq \frac{5\pi^2}{4\left(k+\pi-1\right)^2}=O\left(\frac{1}{k^2}\right).$$ Then for all $t\geq \tau_{22}= \frac{4}{\eta}\log\frac{\norm{a_0-a^*}_2^2}{\delta}=\tilde O\left(\frac{1}{\eta}\log\frac{1}{\delta}\right),$ we have 
$$\norm{a_{t}-a^*}_2^2\leq 5\delta.$$
	\end{lemma}
Combine the above two lemmas together, take $T_2=\tau_{21}+\tau_{22}$, and we complete the proof.
\end{proof}
Theorem \ref{converge} shows that with larger $\eta_w$ than in Stage I, GD converges to the global optimum in polynomial time. Compared to the convergence with constant probability for CNN \citep{du2017gradient},  Assumption \ref{assump1} assures convergence even under arbitrary initialization of $a.$ This partially justifies the importance of shortcut in ResNet. 


\section{Numerical Experiment}\label{sec:experiment}

We present numerical experiments to illustrate the convergence of the GD algorithm. We first demonstrate that with the shortcut prior, our choice of step sizes and the initialization guarantee the convergence of GD. We consider the training of a two-layer non-overlapping convolutional ResNet by solving \eqref{optimization1}. Specifically, we set $p = 8$ and $k \in \{16, 25, 36, 49, 64, 81, 100\}$. The teacher network is set with parameters $a^*$ satisfying $\mathds{1}^\top a^* = \frac{1}{4} \norm{a^*}_2^2$, and $v^*$ satisfying  $v^*_1 = \cos(7\pi/10)$, $v^*_2 = \sin(7\pi/10)$, and $v^*_j = 0$ for $j = 3, \dots, p.$\footnote{$v^*$ essentially satisfies $\angle (v^*, \mathds{1}/\sqrt{p}) = 0.45 \pi$.} More detailed experimental setting is provided in Appendix \ref{expsetting}. We initialize with $w_0 = 0$ and  $a_0$ uniformly distributed over $\BB_0(|\mathds{1}^\top a^*| / \sqrt{k})$. We adopt the following learning rate scheme with Step Size Warmup (SSW) suggested in Section \ref{sec:sketch}: We first choose step sizes $\eta_a = 1/k^2$ and $\eta_w = \eta_a^2$, and run for $1000$ iterations. Then, we choose $\eta_a = \eta_w = 1/k^2$. We also consider learning the same teacher network using step sizes $\eta_w = \eta_a = 1/k^2$ throughout, i.e., without step size warmup. 

We further demonstrate learning the aforementioned teacher network using a student network of the same architecture. Specifically, we keep $a^*, v^*$ unchanged. We use the GD in \citet{du2017gradient} with step size $\eta = 0.1$, and initialize $v_0$ uniformly distributed over the unit sphere and $a$ uniformly distributed over $\BB_0(|\mathds{1}^\top a^*| / \sqrt{k})$.

For each combination of $k$ and $a^*$, we repeat $5000$ simulations for aforementioned three settings, and report the success rate of converging to the global optimum in Table \ref{tab:rate} and Figure \ref{fig:rate}. As can be seen, our GD on ResNet can avoid the spurious local optimum, and converge to the global optimum in all $5000$ simulations. However, GD without SSW can be trapped in the spurious local optimum. The failure probability diminishes as the dimension increases. Learning the teacher network using a two-layer CNN student network  \citep{du2017gradient}  can also be trapped in the spurious local optimum.
\begin{figure}
	\centering
	\includegraphics[width = 0.5\textwidth]{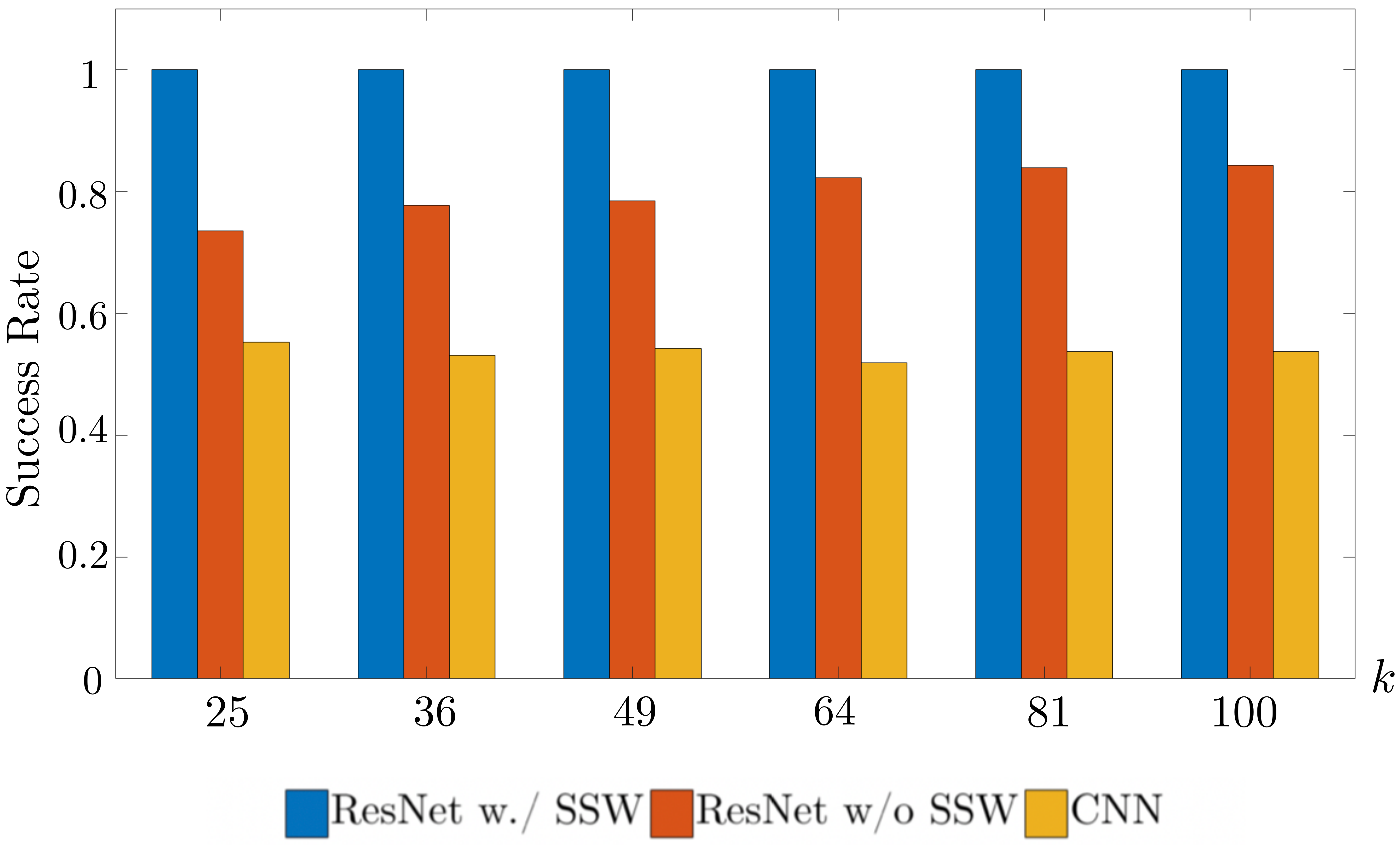}
	\caption{Success rates of converging to the global optimum for GD training ResNet with and without SSW and CNN with varying $k$ and and $p=8.$}
	\label{fig:rate}
\end{figure}
\begin{table*}[htb!]
	\vspace{-0.05in}
	\caption{\it Success rates of converging to the global optimum for GD training ResNet with and without SSW and CNN with varying $k$ and and $p=8$.
		\vspace{-0.05in}
	}
	
	\begin{center}
		\begin{tabular}{lccccccc}
			\hline
			$k$   &16 & 25& 36 & 49 & 64 & 81 & 100 \\
			\hline
			ResNet w/ SSW & 1.0000 & 1.0000& 1.0000 & 1.0000 & 1.0000 & 1.0000 &1.0000 \\
			ResNet w/o SSW &0.7042 & 0.7354 & 0.7776 &0.7848 &0.8220 &0.8388 &0.8426 \\
			CNN &0.5348 &0.5528 &0.5312&0.5426&0.5192&0.5368 &0.5374\\
			\hline
		\end{tabular}
	\end{center}
	\label{tab:rate}
	\vspace{-0.1in}
\end{table*}

We then demonstrate the algorithmic behavior of our GD. We set $k = 25$ for the teacher network, and other parameters the same as in the previous experiment. We initialize $w_0 = 0$ and $a_0 \in \BB_0\left(\frac{|\mathds{1}^\top a^*| }{\sqrt{k}}\right)$. We start with $\eta_a = 1/k^2$ and $\eta_w = \eta_a^2$. After $1000$ iterations, we set the step sizes $\eta_a = \eta_w = 1/k^2$. The algorithm is terminated when $\norm{a_t - a^*}_2^2 + \norm{w_t - w^*}_2^2 \leq 10^{-6}$. We also demonstrate the GD algorithm without SSW at the same initialization. The step sizes are $\eta_a = \eta_w = 1/k^2$ throughout the training.

\begin{figure}
\centering
\includegraphics[width = 0.74\textwidth]{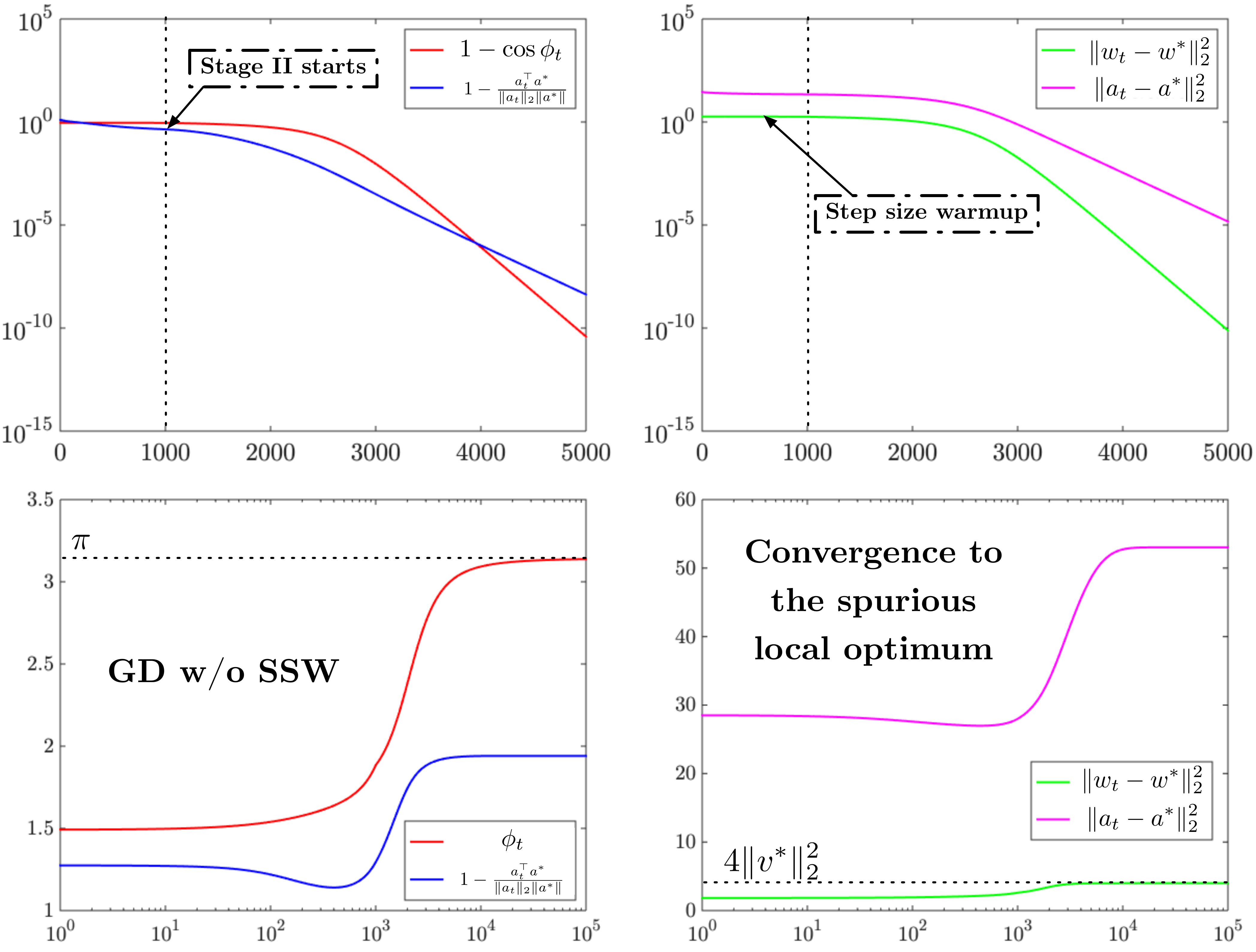}
\caption{Algorithmic behavior of GD on ResNet. The horizontal axis corresponds to the number of iterations.
	}
\label{fig:GD}

\end{figure}
One solution path of GD with SSW is shown in the first row of Figure \ref{fig:GD}. As can be seen, the algorithm has a phase transition. In the first stage, we observe that $w_t$ makes very slow progress due to the small step size $\eta_w$. While $a_t^\top a^*$ gradually increases. This implies the algorithm avoids being attracted by the spurious local optimum. In the second stage, $w_t$ and $a_t$ both continuously evolve towards the global optimum. 

The second row of Figure \ref{fig:GD} illustrates the trajectory of GD without SSW being trapped by the spurious local optimum. Specifically, $(w_t, a_t)$ converges to $(\bar{w}, \bar{a})$ as we observe that $\phi_t$ converges to $\pi$, and $\norm{w_t - w^*}_2^2$ converges to $4\norm{v^*}_2^2$.


\section{Discussions}\label{sec:discuss}

\textbf{Deep ResNet}. Our two-layer network model is largely simplified compared with deep and wide ResNets in practice, where the role of the shortcut connection is more complicated. It is worth mentioning that the empirical results in \citet{veit2016residual} show that ResNet can be viewed as an ensemble of smaller networks, and most of the smaller networks are shallow due to the shortcut connection. They also suggest that the training is dominated by the shallow smaller networks. We are interested in investigating whether these shallow smaller networks possesses similar benign properties to ease the training as our two-layer model.

Moreover, our student network and the teacher network have the same degree of freedom. We have not considered deeper and wider student networks. It is also worth an investigation that what is the role of shortcut connections in deeper and wider networks.

\noindent\textbf{From GD to SGD}. A straightforward extension is to investigate the convergence of SGD with mini-batch. We remark that when the batch size is large, the effect of the noise on gradient is limited and SGD mimics the behavior of GD. When the batch size is small, the noise on gradient plays a significant role in training, which is technically more challenging.



\noindent\textbf{Related Work}. \citet{li2017convergence} study ResNet-type two-layer neural networks with the output weight known  ($a=\mathds{1}$), which is equivalent to assuming $a_t^\top a^* > 0$ for all $t$ in our analysis. Thus, their analysis does not have Stage I ($a_0^\top a^*<0$). Moreover, since they do not need to optimize $a$, they only need to handle the partial dissipativity of $\nabla\cL_w$ with $\delta = 0$ (one-point convexity). In our analysis, however, we also need to handle the the partial dissipativity of $\nabla\cL_a$ with $\delta \neq 0,$ which makes our proof more involved.

\noindent\textbf{Initialization}. Our analysis shows that GD converges to the global optimum, when $w$ is initialized at zero. Empirical results in \citet{li2016demystifying} and \citet{zhang2019fixup} also suggest that deep ResNet works well, when the weights are simply initialized at zero or using the Fixup initialization. We are interested in building a connection between training a two-layer ResNet and its deep counterpart.


\noindent\textbf{Step Size Warmup}. Our choice of step size $\eta_w$ is related to the learning rate warmup and layerwise learning rate in the existing literature. Specifically, \citet{goyal2017accurate} presents an effective learning rate scheme for training ResNet on ImageNet for less than $1$ hour. They start with a small step size, gradually increase (linear scale) it, and finally shrink it for convergence. Our analysis suggests that in the first stage, we need smaller $\eta_w$ to avoid being attracted by the spurious local optimum. This is essentially consistent with \citet{goyal2017accurate}. Note that we are considering GD (no noise), hence, we do not need to shrink the step size in the final stage. While \citet{goyal2017accurate} need to shrink the step size to control the noise in SGD. Similar learning rate schemes are proposed by \citet{smith2017cyclical}.

On the other hand, we incorporate the shortcut prior, and adopt a smaller step size for the inner layer, and a larger step size for the outer layer. Such a choice of step size is shown to be helpful in both deep learning and transfer learning \citep{singh2015layer, howard2018universal}, where it is referred to as differential learning rates or discriminative fine-tuning. It is interesting to build a connection between our theoretical discoveries and these empirical observations.


\bibliographystyle{ims}
\bibliography{ref}
\appendix
\newpage
\noindent\rule[0.5ex]{\linewidth}{4pt}
\begin{center}
\textbf{\LARGE Supplementary Material for Understanding the Importance of Shortcut Connections in ResNet}
\end{center}
\noindent\rule[0.5ex]{\linewidth}{1pt}

\section{Preliminaries}
We first provide the explicit forms of the loss function and its gradients with respect to $w$ and $a.$ 
\begin{proposition}\label{prop:form}
Let $\phi=\angle( {\mathds{1}/\sqrt{p} +{w}},v^*).$When $\norm{\mathds{1}/\sqrt{p}+w}_2=1$, the loss function $\cL\left(w,a\right)$ and the gradient w.r.t $\left(w,a\right)$, i.e., $\nabla_a\cL\left(w,a\right)$ and $\nabla_w\cL\left(w,a\right)$ have the following analytic forms.
\begin{align*}
\cL\left(w,a\right) &=\frac{1}{2}[\frac{\left(\pi-1\right)}{2\pi}\norm{a^*}_2^2
+\frac{\left(\pi-1\right)}{2\pi}\norm{a}_2^2
-\frac{1}{\pi}\left(g\left(\phi\right)-1\right)a^\top a^*\\
&\hspace{2.0in}+\frac{1}{2\pi}\left(\mathds{1}^\top a^*\right)^2
+\frac{1}{2\pi}\left(\mathds{1}^\top a\right)^2
-\frac{1 }{\pi}\mathds{1}^\top a^* a^\top \mathds{1}],\\
\nabla_{a} \cL(a, w) & = \frac{1}{2\pi} (\mathds{1}\mathds{1}^\top+(\pi-1))a-\frac{1}{2\pi}(\mathds{1}\mathds{1}^\top+(g(\phi)-1))a^*,\notag \\
 \nabla_w\cL\left(w,a\right)
&=-\frac{a^\top a^* \left(\pi -\phi\right)}{2\pi}\left(I-(\mathds{1}/\sqrt{p}+w)(\mathds{1}/\sqrt{p}+w)^\top\right)v^*,
\end{align*}
where $g(\phi)=(\pi-\phi)\cos(\phi)+\sin(\phi).$
\end{proposition}
This proposition is a simple extension of Theorem 3.1 in \citet{du2017gradient}. Here, we omit the proof.

For notational simplicity, we denote $v_t=\mathds{1}/\sqrt{p}+w_t$ in the future proof.
\section{Proof of Theoretical Results}

\subsection{Proof of Proposition \ref{pf:spurious}}\label{pf:spurious}
\begin{proof}
Recall that \citet{du2017gradient} proves that $(\bar{v},\bar{a})=(-v^*, (\mathds{1}\mathds{1}^\top+(\pi-1)I)^{-1}(\mathds{1}\mathds{1}^\top-I)a^*)$ is the spurious local optimum of the CNN counterpart to our ResNet. Substitute $\bar{v}$ by $\frac{\mathds{1}/\sqrt{p} + w}{\norm{\mathds{1}/\sqrt{p} + w}_2} $ and we prove the result.
\end{proof}

\subsection{Proof of Theorem \ref{stage1_a}}\label{pf_stage1_a}
\subsubsection{Proof of Lemma \ref{lem_sum_a}}
\begin{proof}
By simple manipulication, we know that the initialization of $a$ satisfies  $-2\left(\mathds{1}^\top{a}^*\right)^2\leq\mathds{1}^\top a^*\mathds{1}^\top a_0 - \left(\mathds{1}^\top a^*\right)^2\leq 0.$
	We first prove the right side of the inequality. Expand $a_{t}$ as $a_{t-1} - \eta_a\nabla_a\cL(w_{t-1},a_{t-1}),$ and we have
	\begin{align*}
	\mathds{1}^\top a^*\mathds{1}^\top a_{t}&=\left(1-\frac{\eta_a \left(k+\pi-1\right)}{2\pi}\right)\mathds{1}^\top a^*\mathds{1}^\top a_{t-1}
	+ \frac{\eta_a \left(k+g\left(\phi_{t-1}\right)-1\right)}{2\pi}\left(\mathds{1}^\top a^*\right)^2\\
	&\leq \left(1-\frac{\eta_a \left(k+\pi-1\right)}{2\pi}\right)\mathds{1}^\top a^*\mathds{1}^\top a_{t-1}
	+ \frac{\eta_a \left(k+\pi-1\right)}{2\pi}\left(\mathds{1}^\top a^*\right)^2.
	\end{align*}
	Subtract $\left(\mathds{1}^\top a^*\right)^2$ from both sides, then we get
	\begin{align*}
	\mathds{1}^\top a^*\mathds{1}^\top a_{t}-\left(\mathds{1}^\top a^*\right)^2
	&\leq  \left(1-\frac{\eta_a \left(k+\pi-1\right)}{2\pi}\right)(\mathds{1}^\top a^*\mathds{1}^\top a_{t-1}-\left(\mathds{1}^\top a^*\right)^2)\\
	&\leq   \left(1-\frac{\eta_a \left(k+\pi-1\right)}{2\pi}\right)^t(\mathds{1}^\top a^*\mathds{1}^\top a_{0}-\left(\mathds{1}^\top a^*\right)^2)\leq 0,
	\end{align*}
for any $t\geq 1.$ The right side inequality is proved.

	The proof of the left side follows similar lines. Since $g(\phi)\geq0,$ we have 
	\begin{align*}
	\mathds{1}^\top a^*\mathds{1}^\top a_{t}&=\left(1-\frac{\eta_a \left(k+\pi-1\right)}{2\pi}\right)\mathds{1}^\top a^*\mathds{1}^\top a_{t-1}
	+ \frac{\eta_a \left(k+g\left(\phi_{t-1}\right)-1\right)}{2\pi}\left(\mathds{1}^\top a^*\right)^2\\
	&\geq \left(1-\frac{\eta_a \left(k+\pi-1\right)}{2\pi}\right)\mathds{1}^\top a^*\mathds{1}^\top a_{t-1}
	+ \eta_a \frac{ k-1}{2\pi}\left(\mathds{1}^\top a^*\right)^2,
\end{align*}
which is equivalent to the following inequality.
\begin{align*}
\mathds{1}^\top a^*\mathds{1}^\top a_{t}-\left(\mathds{1}^\top a^*\right)^2
	&\geq  \left(1-\frac{\eta_a \left(k+\pi-1\right)}{2\pi}\right)(\mathds{1}^\top a^*\mathds{1}^\top a_{t-1}-\left(\mathds{1}^\top a^*\right)^2)-\frac{\eta_a }{2}\left(\mathds{1}^\top a^*\right)^2.\\
	&\geq \left(1-\frac{\eta_a \left(k+\pi-1\right)}{2\pi}\right)^t(\mathds{1}^\top a^*\mathds{1}^\top a_{0}-\left(\mathds{1}^\top a^*\right)^2)-\frac{1}{1-\left(1-\frac{\eta_a \left(k+\pi-1\right)}{2\pi}\right)}\frac{\eta_a }{2}\left(\mathds{1}^\top a^*\right)^2\\
	&\geq \left(1-\frac{\eta_a \left(k+\pi-1\right)}{2\pi}\right)^t(\mathds{1}^\top a^*\mathds{1}^\top a_{0}-\left(\mathds{1}^\top a^*\right)^2)-\frac{\pi}{k+\pi-1}\left(\mathds{1}^\top a^*\right)^2\\
	&\geq \left(1-\frac{\eta_a \left(k+\pi-1\right)}{2\pi}\right)^t(-2\left(\mathds{1}^\top a^*\right)^2)-\frac{\pi}{k+\pi-1}\left(\mathds{1}^\top a^*\right)^2\\
	&\geq -3\left(\mathds{1}^\top a^*\right)^2.
	\end{align*}
	Then we prove the lemma.
\end{proof}

\subsubsection{Proof of Lemma \ref{lem_small}}
\begin{proof}
For each iteration, the distance of $w_t$ moving towards $\bar{w}$ is upper bounded by the product of the step size $\eta_w$ and the norm of the gradient $ \nabla_w\cL\left(w,a\right).$ We first bound the norm of the gradient. From the analytic form of $ \nabla_w\cL\left(w,a\right),$  we need to bound $a^\top a^*.$ We first have the following lower bound.
\begin{align*}
a_{t+1}a^*&= \left(1-\frac{\eta_a\left(\pi-1\right)}{2\pi}\right)a_t^\top a^* + \frac{\eta_a\left(g\left(\phi_t\right)-1\right)}{2\pi}\norm{a^*}_2^2
	+\frac{\eta_a}{2\pi}\left(\left(\mathds{1}^\top a^*\right)^2 -\mathds{1}^\top a^*\mathds{1}^\top a_t\right)\\
	&\geq \left(1-\frac{\eta_a\left(\pi-1\right)}{2\pi}\right)a_t^\top a^* -\eta_a \frac{2}{\pi}\norm{a^*}_2^2,
\end{align*}
which is equivalent to
\begin{align*}
a_{t+1}a^*+\frac{4}{\pi-1}\norm{a^*}_2^2
	&\geq \left(1-\frac{\eta_a\left(\pi-1\right)}{2\pi}\right)(a_t^\top a^* + \frac{4}{\pi-1}\norm{a^*}_2^2)\\
	&\geq  \left(1-\frac{\eta_a\left(\pi-1\right)}{2\pi}\right)^{t+1}(a_0^\top a^* + \frac{4}{\pi-1}\norm{a^*}_2^2).
\end{align*}
Since $a_0^\top a^*\geq-\norm{a^*}_2^2,$ we have $a_0^\top a^* + \frac{4}{\pi-1}\norm{a^*}_2^2\geq 0.$ Thus,  when $\eta_a<\frac{2\pi}{\pi-1},$ $$a_{t+1}a^*\geq-\frac{4}{\pi-1}\norm{a^*}_2^2\geq -2\norm{a^*}_2^2.$$
When $a_{t+1}a^*<2\norm{a^*}_2^2,$ the following inequality holds true.
 \begin{align*}
		\norm{\nabla_wL\left(w_t,a_t\right)}_2^2=&\frac{\left(a_t^\top a^*\right)^2 \left(\pi -\phi_t\right)^2}{4\pi^2}v^{*\top}\left(I-v_tv_t^\top\right)v^*\\&\leq \norm{a^*}_2^4(I-v_t^\top v^*)(I+v_t^\top v^*)\leq\norm{a^*}_2^4\norm{v_t-v^*}_2^2.
	\end{align*}
We next prove that when $\eta_w$ is small enough, $\phi_t<\pi/2$ holds for all $t\leq T=O(1/\eta_a^2).$ 
	We first have the following inequality.
	\begin{align*}
		1\leq\norm{\tilde{v}_{t+1}}_2=\sqrt{\norm{v_t}_2^2+\norm{\eta_w\nabla_wL\left(w_t,a_t\right)}_2^2}
		\leq 1+\norm{\eta_w\nabla_wL\left(w_t,a_t\right)}_2.
	\end{align*}
Under Assumption \ref{assump1}, we know that  $\phi_0<\pi/3.$ Then we can bound the norm of the difference between iterates $w_{t+1}$ and $w^*$  as follows.
\begin{align*}
\norm{{v}_{t+1}-v^*}_2&=\norm{\tilde{v}_{t+1}/\norm{\tilde{v}_{t+1}}_2-v^*}_2\leq \frac{1}{\norm{\tilde{v}_{t+1}}_2}\norm{\tilde{v}_{t+1}-v^*}_2+1-\frac{1}{\norm{\tilde{v}_{t+1}}_2}\\
&\leq \norm{\tilde{v}_{t+1}-v^*}_2+1-\frac{1}{1+\norm{\eta_w\nabla_wL\left(w_t,a_t\right)}_2}.
\end{align*}
Plug in the upper bound of the norm of $\nabla_wL\left(w_t,a_t\right),$ and we obtain
\begin{align*}
\norm{{v}_{t+1}-v^*}_2&\leq\norm{\tilde{v}_{t+1}-v^*}_2+1-\frac{1}{1+\eta_w\norm{a^*}_2^2\norm{v_t-v^*}_2}\\
&=\norm{v_t-v^*-\eta_w\nabla_w\cL(a_t,w_t)}_2^2+\frac{\eta_w\norm{a^*}_2^2\norm{v_t-v^*}_2}{1+\eta_w\norm{a^*}_2^2\norm{v_t-v^*}_2}\\
&\leq\norm{v_t-v^*}_2+\eta_w\norm{\nabla_w\cL(a_t,w_t)}_2+\eta_w\norm{a^*}_2^2\norm{v_t-v^*}_2\\
&\leq \norm{v_t-v^*}_2+\eta_w\norm{a^*}_2^2\norm{v_t-v^*}_2+\eta_w\norm{a^*}_2^2\norm{v_t-v^*}_2\\
&=(1+2\eta _w\norm{a^*}_2^2)\norm{v_t-v^*}_2\leq(1+2\eta _w\norm{a^*}_2^2)^t \norm{v_0-v^*}_2\\
&\leq \exp( 2t\eta_w\norm{a^*}_2^2)\norm{v_0-v^*}_2\\
&\leq \exp( 2t\eta_w\norm{a^*}_2^2)\leq 2-2\cos\left(\frac{5}{12}\pi\right),
\end{align*}
for all $t\leq T=O(1/\eta_a^2),$ when $\eta_w=C_1\norm{a^*}_2^2\eta_a^2=\tilde O(\eta_a^2)$ for some constant $C_1>0.$ Thus $\phi_t\leq \frac{5}{12}\pi$ for all $t\leq T=O(1/\eta_a^2).$
\end{proof}

\subsubsection{Proof of Lemma \ref{thm_pd_a}}\label{pf_thm_pd}
\begin{proof}
	For any $C_3\in(0,1),$ if we have $a^\top a^*\leq {C_3}\norm{a^*}_2^2$, the norm of the difference between $a$ and $a^*$  satisfies the following inequality.  $$\norm{a-a^*}_2^2\geq\left(1-{2C_3}\right)\norm{a^*}_2^2.$$ Let $C_2=g(\frac{5}{12}\pi)-1=0.4402.$ SInce $\phi\leq \frac{5}{12}\pi,$ and $g$ is strictly decreasing,  we know that $g(\phi)\geq C_2.$ Using the above two inequalities, we can lower bound the inner product between the negative gradient and the difference between $a$ and $a^*$ as follows.
	\begin{align*}
	\langle-\nabla_aL\left(w+\xi,a+\epsilon\right),a^*-a\rangle
	&= \frac{1}{2\pi}\left(\mathds{1}^\top a- \mathds{1}^\top a^*\right)^2
	+\frac{1}{2\pi}\left(\left(\pi-1\right)a-\left(g\left(\phi\right)-1\right)a^*\right)^\top \left(a-a^*\right)\\
	&=\frac{1}{2\pi}\left(\mathds{1}^\top a- \mathds{1}^\top a^*\right)^2
	+\frac{1}{2\pi}\left(\pi-g\left(\phi\right)\right)a^\top \left(a-a^*\right)+\frac{ g\left(\phi\right)-1}{2\pi}\norm{a-a^*}_2\\
	&\geq -\frac{1}{2\pi}\left(\pi- g\left(\phi\right)\right)a^\top a^*+\frac{ g\left(\phi\right)-1}{2\pi}\norm{a-a^*}_2^2\\
	&\geq -\frac{1}{2\pi}\left(\pi- g\left(\phi\right)\right)a^\top a^*+\frac{g\left(\phi\right)-1}{4\pi}\norm{a-a^*}_2^2+\frac{g\left(\phi\right)-1}{4\pi}\norm{a-a^*}_2^2\\	
	&\geq-\frac{C_3}{2}\norm{a^*}_2^2 +\frac{C_2}{4\pi }\left(1-{2C_3}\right)\norm{a^*}_2^2
	+   \frac{C_2}{4\pi }\norm{a-a^*}_2^2\\
	&\geq  \frac{C_2}{4\pi }\norm{a-a^*}_2^2\geq   \frac{1}{10\pi }\norm{a-a^*}_2^2,
	\end{align*}
	when $C_3\leq\frac{C_2}{2(C_2+\pi)}.$ Take $C_3=\frac{1}{20},$ and we prove the result.
\end{proof}

\subsubsection{Proof of Lemma \ref{lem_escape}}\label{pf_lem_escape}
\begin{proof}
We prove the result by contradiction. Specifically, we show that if $a_t\in\cA$ always holds, there always exist some time $\tau$ such that $a_\tau\notin\cA,$ which is a contradiction. Formally, suppose $\forall\tau\leq t, a_\tau\in \cA,$ then we have
	\begin{align}\label{eq:2}
	\norm{a_{t+1}-a^*}_2^2
	&=\norm{a_t-a^*}_2^2
	-2\langle-\eta_a\mathbb{E}_{\xi,\epsilon}\nabla_aL\left(w_t,a_t\right),a^*-a_t\rangle\\
	&+\norm{\eta_a \nabla_aL\left(w_t,a_t\right)}_2^2.
	\end{align}
The second term is lower bounded according to the partial dissipativity of $\nabla_a \cL$. Thus, we only need to bound the norm of the gradient.
	\begin{align*}
	\norm{\nabla_aL\left(w_t,a_t\right)}_2^2
	&=\norm{\nabla_aL\left(w_t,a_t\right)-\nabla_aL\left(w^*,a^*\right)}_2^2\\
	&=
	\norm{\frac{1}{2\pi}\left(\mathds{1}\mathds{1}^\top+\left(\pi-1\right)I\right)\left(a_t-a^*\right)
		-\frac{g\left(\phi\right)-\pi}{2\pi}a^*}_2^2\\
	&\leq \frac{1}{2\pi^2}\norm{\left(\mathds{1}\mathds{1}^\top+\left(\pi-1\right)I\right)\left(a_t-a^*\right)}_2^2 
	+\frac{1}{2}\norm{a^*}_2^2\\
	&\leq \frac{\left(k+\pi-1\right)^2}{\pi^2}\left(\norm{a_t-a^*}_2^2\right) 
	+\frac{1}{2}\norm{a^*}_2^2.
	\end{align*}
	Plug the above bound into \eqref{eq:2},  then we have
	\begin{align*}
	\norm{a_{t+1}-a^*}_2^2
		&\leq \left(1-\frac{\pi}{5}\eta_a  + \eta_a^2\frac{\left(k+\pi-1\right)^2}{\pi^2}\right)\norm{a_t-a^*}_2^2 +\frac{\eta_a ^2}{2}\norm{a^*}_2^2\\
	&\leq\left(1-\lambda_1\right)\norm{a_t-a^*}_2^2+ b_1\\
	&\leq(1-\lambda_1)^{t+1}\norm{a_0-a^*}_2^2+\frac{b_1}{\lambda_1},
	\end{align*}
	where  $\lambda_1=\frac{\pi}{5}\eta_a - \eta_a^2\frac{\left(k+\pi-1\right)^2}{\pi^2}$ and $b_1=\frac{\eta_a ^2}{2}\norm{a^*}_2^2.$ When $\eta_a <\frac{\pi}{20(k+\pi-1)^2},$ we have $\frac{b_1}{\lambda_1}\leq \frac{\norm{a^*}_2^2}{6}.$ Thus,  after $\tau_{11}=O(\frac{1}{\eta_a })$ iterations, we have 
	\begin{align*}
	\norm{a_{\tau_{11}}-a^*}_2^2&< \frac{\norm{a^*}_2^2}{4}.
	\end{align*}
	On the other hand, $a_{\tau_{11}}\in \cA$ implies that $\norm{a_{\tau_{11}}-a^*}_2^2\geq \frac{1}{4}\norm{a^*}_2^2$. Thus,  after $\tau_{11}=O(\frac{1}{\eta_a})$ iterations, we have
	 $$\frac{1}{20}\norm{a^*}_2^2\leq a_t^\top a^*~\text{and } \norm{a_{t}-a^*/2}_2^2\leq\norm{a^*}_2^2.$$ 	Moreover, $\norm{a_{t}-a^*/2}_2^2\leq\norm{a^*}_2^2$ implies $a_t^\top a^*\leq2\norm{a^*}_2^2,$ and we prove the lemma.	
\end{proof}
\subsubsection{Proof of Lemma \ref{lem_inner_a}}
\begin{proof} 
	We first prove the left side. Write $a_{t+1} = a_t - \eta_a\nabla_a\cL(w,a)$ and we have
	\begin{align*}
	a_{t+1}^\top a^*
	&= \left(1-\frac{\eta_a \left(\pi-1\right)}{2\pi}\right)a_t^\top a^* + \frac{\eta_a \left(g\left(\phi_t\right)-1\right)}{2\pi}\norm{a^*}_2^2
	+\frac{\eta_a }{2\pi}\left(\left(\mathds{1}^\top a^*\right)^2 -\mathds{1}^\top a^*\mathds{1}^\top a_t\right)\\
	&\geq \left(1-\frac{\eta_a \left(\pi-1\right)}{2\pi}\right)a_t^\top a^* + 
	\eta_a  \frac{C_2}{2\pi}\norm{a^*}_2^2.
\end{align*}
The last inequality holds since $g(\phi)\geq 1$ and $\left(\mathds{1}^\top a^*\right)^2 -\mathds{1}^\top a^*\mathds{1}^\top a_t\geq0.$
Subtract $\frac{C_2}{\pi-1}\norm{a^*}_2^2$ from both sides and we have  the following inequality
\begin{align*}
a_{t+1}^\top a^*-\frac{C_2}{\pi-1}\norm{a^*}_2^2&\geq \left(1-\frac{\eta_a \left(\pi-1\right)}{2\pi}\right)\left(a_t^\top a^*-\frac{C_2}{\pi-1}\norm{a^*}_2^2\right)\\
	&\geq\left(1-\frac{\eta_a \left(\pi-1\right)}{2\pi}\right)^t\left(a_0^\top a^*-\frac{C_2}{\pi-1}\norm{a^*}_2^2\right).
	\end{align*}
	Thus, when $t\geq\tau_{12}=\tilde O(1/\eta_a )>0,$ we have 
	$a_{t}^\top a^*\geq \frac{1}{5}\norm{a^*}_2^2.$ 
	
	For the right side, follows similar lines to the left side, we have
	\begin{align*}
	a_{t+1}^\top a^*
	&= \left(1-\frac{\eta_a \left(\pi-1\right)}{2\pi}\right)a_t^\top a^* + \frac{\eta_a \left(g\left(\phi_t\right)-1\right)}{2\pi}\norm{a^*}_2^2
	+\frac{\eta_a }{2\pi}\left(\left(\mathds{1}^\top a^*\right)^2 -\mathds{1}^\top a^*\mathds{1}^\top a_t\right)\\
	&\leq \left(1-\frac{\eta_a \left(\pi-1\right)}{2\pi}\right)a_t^\top a^* + \eta_a  \frac{\pi-1}{2\pi}\norm{a^*}_2^2+\eta_a  \frac{3}{2\pi}(\mathds{1}^\top a^*)^2\\
	&\leq \left(1-\frac{\eta_a \left(\pi-1\right)}{2\pi}\right)^{t+1}a_0^\top a^*+\norm{a^*}_2^2+\frac{3}{\pi-1}(\mathds{1}^\top a^*)^2.
\end{align*}
Note that $a_0^\top a^*\leq 2\norm{a^*}_2^2.$
Thus, for all $t$,  $a_{t+1}^\top a^*\leq3\norm{a^*}_2^2+2(\mathds{1}^\top a^*)^2.$ 
\end{proof}
\subsection{proof of Theorem \ref{large_etaw}}\label{pf_large_etaw}
\subsubsection{Proof of Lemma \ref{lem_pd_w}}
\begin{proof}
Note that $\norm{v_t}_2=\norm{v^*}_2=1,$ according to Proposition \ref{prop:form}, the gradient with respect to $w$ can be rewritten as 
	\begin{align*}
	\nabla_w\cL\left(w_t,a_t\right)&=-\frac{a_t^\top a_t^* \left(\pi -\phi_t\right)}{2\pi}\left(I-v_tv_t^\top\right)v^*.
	\end{align*}
	Then we have the following inequality.
	\begin{align*}
	\langle-\nabla_w\cL(w_t,a_t),w^*-w_t\rangle&=	\langle-\nabla_w\cL(w_t,a_t),v^*-v_t\rangle\\
	&=\frac{a_t^\top a_t^* \left(\pi -\phi_t\right)}{2\pi}\left(1- (v_t^\top v^*)^2\right)\\
	&\geq\frac{m}{4}(1-v_t^\top v^*)\\
	&=\frac{m}{8}\norm{w-{w}^*}_2^2.
	\end{align*}
\end{proof}
\subsubsection{proof of Theorem \ref{large_etaw}}\label{pf_large_etaw}
\begin{proof}
First,  we  bound the norm of the gradient as follows
	\begin{align*}
		\norm{\nabla_wL\left(w,a\right)}_2^2=&\frac{\left(a^\top a^*\right)^2 \left(\pi -\phi\right)^2}{4\pi^2}v^{*\top}\left(I-vv^\top\right)v^*\leq \frac{M^2}{4}(I-v^\top v^*)(I+v^\top v^*)\leq \frac{M^2}{4}\norm{v-v^*}_2^2.
	\end{align*}
Next we show that 
	$\norm{v_{t+1}-v^*}_2^2 \leq \norm{\tilde{v}_{t+1}-v^*}_2^2$. We first have the following two inequalities.
	\begin{align*}
		\norm{\tilde{v}_{t+1}}_2^2=\norm{v_t}_2^2+\norm{\eta_w\nabla_wL\left(w_t,a_t\right)}_2^2
		\geq 1.
	\end{align*}
	\begin{align*}
	\tilde{v}_{t+1}^\top v^*=v_t^\top v^*+\eta_w\langle-\nabla_w\cL(w_t+\xi,a_t+\epsilon),{v}^*-v_t\rangle \geq v_t^\top v^*>0.
	\end{align*}

Thus,  $0<v_{t+1}^\top v^*\leq 1.$ We  then have 
	\begin{align*}
		\norm{\tilde{v}_{t+1}-v^*}_2^2
		&= 1 + \norm{\tilde{v}_{t+1}}_2^2-2\norm{\tilde{v}_{t+1}}_2v_{t+1}^\top v^*\\
		&\geq 1 + 1-2w_{t+1}^\top w^*=\norm{v_{t+1}-v^*}_2^2.
	\end{align*}

Then the distance between $\tilde{w}_{t+1}$ and $w^*$ is as follows.
\begin{align*}
\norm{v_{t+1}-v^*}_2^2\leq\norm{\tilde{v}_{t+1}-v^*}_2^2&=\norm{w_t-\eta_w\nabla_w\cL(a_t,w_t)-w^*}_2^2\\
&=\norm{v_t-v^*}_2^2+\norm{\eta_w\nabla_w\cL(a_t,w_t)}_2^2-2\langle-\nabla_w\cL(w_t+\xi,a_t+\epsilon),{v}^*-v_t\rangle\\
& \leq (1-\eta_w\frac{m}{4}+\eta_w^2  \frac{M^2}{4})\norm{v_t-{v}^*}_2^2\leq\norm{v_t-{v}^*}_2^2,
\end{align*}	
when $\eta_w\leq \frac{m}{M^2}.$ Thus, $\phi_t\leq \phi_0\leq \frac{5}{12}\pi.$ We prove the first part.

We then prove the second part. Using the same expansion as in Lemma \ref{lem_inner_a}, we get
\begin{align*}
	a_{t+1}^\top a^*
	&= \left(1-\frac{\eta_a\left(\pi-1\right)}{2\pi}\right)a_t^\top a^* + \frac{\eta_a\left(g\left(\phi_t\right)-1\right)}{2\pi}\norm{a^*}_2^2
	+\frac{\eta_a}{2\pi}\left(\left(\mathds{1}^\top a^*\right)^2 -\mathds{1}^\top a^*\mathds{1}^\top a_t\right)\\
	&\geq \left(1-\frac{\eta_a\left(\pi-1\right)}{2\pi}\right)a_t^\top a^* + 
	\eta_a \frac{C_2}{2\pi}\norm{a^*}_2^2.
\end{align*}
Choose $\eta_a<\frac{2\pi}{\pi-1},$ such that $1-\frac{\eta_a\left(\pi-1\right)}{2\pi}<1.$ If $m\leq a_t^\top a^*\leq \frac{C_2}{\pi-1}\norm{a^*}_2^2,$ the following inequality shows that $a_{t}^\top a^*$ increases over time.
\begin{align*}
	a_{t+1}^\top a^*
	&\geq \left(1-\frac{\eta_a\left(\pi-1\right)}{2\pi}\right)a_t^\top a^* + 
	\eta_a \frac{C_2}{2\pi}\norm{a^*}_2^2\geq a_t^\top a^*\geq m.
\end{align*}
If $a_t^\top a^*\geq \frac{C_2}{\pi-1}\norm{a^*}_2^2,$ we show that this inequality holds for all $t.$
\begin{align*}
	a_{t+1}^\top a^*
	&\geq \left(1-\frac{\eta_a \left(\pi-1\right)}{2\pi}\right)a_t^\top a^* + 
	\eta_a \frac{C_2}{2\pi}\norm{a^*}_2^2,\\
	& \geq \left(1-\frac{\eta_a\left(\pi-1\right)}{2\pi}\right) \frac{C_2}{\pi-1}\norm{a^*}_2^2+ 
	\eta_a \frac{C_2}{2\pi}\norm{a^*}_2^2= \frac{C_2}{\pi-1}\norm{a^*}_2^2
\end{align*}
Combine these two cases together,  we have $a_{t+1}^\top a^*\geq \min\{m, \frac{C_2}{\pi-1}\norm{a^*}_2^2\}=m.$
The other side follows similar lines in Lemma \ref{lem_inner_a}. Here, we omit the proof.
\end{proof}

\subsection{Proof of Theorem \ref{converge}}\label{pf_converge}
\subsubsection{Proof of Lemma \ref{new_pda}}
\begin{proof}
Note that $$\norm{{w}_{t}-w^*}_2^2\leq \delta~\iff~\cos(\phi_t)\geq1-\frac{\delta}{2}.$$
	Moreover, we can  bound $g(\phi_t)$ as follows
	\begin{align*}
	\pi\geq g(\phi_t)&=(\pi-\phi_t)\cos\phi_t+\sin\phi_t\geq \left(1-\frac{\delta}{2}\right)\pi=\pi- \frac{\delta}{2}\pi.
	\end{align*}
Thus we have the partial dissipativity of $\nabla_a\cL.$
\begin{align*}
	\langle-\nabla_aL\left(w,a\right),a^*-a\rangle
	&= \frac{1}{2\pi}\left(\mathds{1}^\top a- \mathds{1}^\top a^*\right)^2+\frac{1}{2\pi}\left(\left(\pi-1\right)a-\left( g\left(\phi\right)-1\right)a^*\right)^\top \left(a-a^*\right)\\
	&=\frac{1}{2\pi}\left(\mathds{1}^\top a- \mathds{1}^\top a^*\right)^2
	+\frac{1}{2\pi}\left(\pi- g\left(\phi\right)\right)a^{*\top} \left(a-a^*\right)+\frac{\pi-1}{2\pi}\norm{a-a^*}_2^2\\
	&\geq \frac{\pi-1}{2\pi}\norm{a-a^*}_2^2 -\delta/5.
	\end{align*}
	\end{proof}
\subsubsection{Proof of Lemma \ref{converge_w}}
\begin{proof}
First,   we bound  the norm of the gradient as follows
	\begin{align*}
		\norm{\nabla_wL\left(w,a\right)}_2^2=&\frac{\left(a^\top a^*\right)^2 \left(\pi -\phi\right)^2}{4\pi^2}v^{*\top}\left(I-vv^\top\right)v^*\leq \frac{M^2}{4}(I-v^\top v^*)(I+v^\top v^*)\leq \frac{M^2}{4}\norm{v-v^*}_2^2
	\end{align*}
	We  next show that 
	$\norm{w_{t+1}-w^*}_2^2 \leq \norm{\tilde{w}_{t+1}-w^*}_2^2$. We first have the following inequality.
	\begin{align*}
		\norm{\tilde{w}_{t+1}}_2^2=\norm{w_t}_2^2+\norm{\eta\nabla_wL\left(w_t,a_t\right)}_2^2
		\geq 1.
	\end{align*}
Since we have $w_{t+1}^\top w^*\leq 1$, we show that $\norm{\tilde{v}_{t+1}-v^*}_2^2\leq \norm{v_{t+1}-v^*}_2^2.$
	\begin{align*}
		\norm{\tilde{v}_{t+1}-v^*}_2^2
		&= 1 + \norm{\tilde{w}_{t+1}}_2^2-2\norm{\tilde{w}_{t+1}}_2w_{t+1}^\top w^*\\
		&\geq 1 + 1-2w_{t+1}^\top w^*=\norm{v_{t+1}-v^*}_2^2.
	\end{align*}

Then the distance between $\tilde{w}_{t+1}$ and $w^*$ is as follows.
\begin{align*}
\norm{v_{t+1}-v^*}_2^2\leq\norm{\tilde{v}_{t+1}-v^*}_2^2&=\norm{w_t-\eta\nabla_w\cL(a_t,w_t)-w^*}_2^2\\
&=\norm{w_t-w^*}_2^2+\norm{\eta\nabla_w\cL(a_t,w_t)}_2^2-2\langle-\nabla_w\cL(w+\xi,a+\epsilon),{w}^*-w\rangle\\
& \leq (1-\eta\frac{m}{4}+\eta^2  \frac{M^2}{4})\norm{v_t-{v}^*}_2^2.
\end{align*}	
So we have for any $t,$
\begin{align*}
\norm{{v}_{t}-v^*}_2^2& \leq (1-\eta\frac{m}{4}+\eta^2  \frac{M^2}{4})^t\norm{v_0-{v}^*}_2^2.
\end{align*}
Thus, choose $\eta\leq\frac{m}{2M^2}=\tilde O(\frac{1}{k^2}),$ and after $t\geq\tau_{21}=\frac{4}{m\eta}\log\frac{4}{\delta}$ iterations, we have  
$$\norm{{v}_{t}-v^*}_2^2\leq \delta,$$ which is equivalent to $$\norm{{w}_{t}-w^*}_2^2\leq \delta.$$
\end{proof}

\subsubsection{Proof of Lemma \ref{converge_a}}
\begin{proof}
The proof follows similar lines to that of Lemma \ref{converge_w}. By the partial dissipativity  of $\cL_a,$ we have
\begin{align*}
\norm{a_{t+1}-a^*}_2^2
&=\norm{a_t-a^*}_2^2
-2\langle-\eta\mathbb{E}_{\xi,\epsilon}\nabla_aL\left(w_t,a_t\right),a^*-a_t\rangle\\
&+\norm{\eta \nabla_aL\left(w_t,a_t\right)}_2^2\\
&\leq \left(1-\eta \frac{\pi-1}{\pi} + \eta^2\frac{\left(k+\pi-1\right)^2}{\pi^2}\right)\norm{a_t-a^*}_2^2 +2{\eta^2\delta^2/25}+2\eta \delta/5\\
&\leq\left(1-\lambda_2\right)\norm{a_t-a^*}_2^2+ \frac{4}{5}\eta \delta\\
&\leq(1-\lambda_2)^{t+1}\norm{a_0-a^*}_2^2+\frac{b_2}{\lambda_2}.
\end{align*}
where $\lambda_2=\eta \frac{\pi-1}{\pi} -\eta^2\frac{\left(k+\pi-1\right)^2}{\pi^2}$ and $b_2=\frac{4}{5}\eta\delta.$ Take $\eta\leq \frac{5\pi^2}{4\left(k+\pi-1\right)^2},$ and then $\lambda_2\geq \frac{\eta}{4}.$ When $t\geq \tau_{22}=\frac{4}{\eta}\log\frac{\norm{a_0-a^*}_2^2}{\delta}=\tilde O(\frac{1}{\eta}\log\frac{1}{\delta}),$ we have 
$$\norm{a_{t}-a^*}_2^2\leq 5\delta.$$
\end{proof}

\section{Experimental Settings}\label{expsetting}
The output weight $a^*$ in the teacher network is chosen as in Table \ref{table:a}.
\begin{table*}[!htb]

\centering
\begin{tabular}{c | c}
\hline
& \\[-1em]
$k$ & $(a^*)^\top$ \\\hline
& \\[-1em]
$16$ & $[\underbrace{1, \dots, 1}_{9}, \underbrace{-1, \dots, -1}_{7}]$ \\\hline
& \\[-1em]
$25$ & $[\underbrace{1, \dots, 1}_{14}, \underbrace{-1, \dots, -1}_{11}]$ \\\hline
& \\[-1em]
$36$ & $[\underbrace{1, \dots, 1}_{19}, \underbrace{-1, \dots, -1}_{16}, 0]$  \\\hline
& \\[-1em]
$49$ & $[\underbrace{1, \dots, 1}_{26}, \underbrace{-1, \dots, -1}_{22}, 0]$ \\\hline
& \\[-1em]
$64$ & $[\underbrace{1, \dots, 1}_{34}, \underbrace{-1, \dots, -1}_{30}]$ \\\hline
& \\[-1em]
$81$ & $[\underbrace{1, \dots, 1}_{43}, \underbrace{-1, \dots, -1}_{38}]$ \\\hline
& \\[-1em]
$100$ & $[\underbrace{1, \dots, 1}_{52}, \underbrace{-1, \dots, -1}_{47}, 0]$  \\\hline
\end{tabular}
\caption{Output weight $a^*.$}
\label{table:a}
\end{table*}

The trajectories in Figure \ref{fig:GD} are obtained with $a$ initialized at
\begin{align*}
a_0 = & [-0.1268, -0.1590, -0.1071, -0.1594, -0.4670, 0.1563, 0.1894, -0.2390, -0.0602, \\
& -0.5047, 0.0325, -0.0886, 0.1514, -0.0883, -0.0243, 0.1198, -0.2805, 0.0024, \\
& -0.0855, 0.0742, -0.0976, -0.1768, 0.1207, 0.0049, 0.1809].
\end{align*}

\end{document}